\documentclass{article}

\usepackage{amsmath,amsfonts,bm}
\usepackage{derivative}

\newcommand{\deriv}{{\mathrm{d}}}

\def\eqref#1{equation~\ref{#1}}

\def\1{\bm{1}}

\def\rvv{{\mathbf{v}}}
\def\rvw{{\mathbf{w}}}
\def\rvx{{\mathbf{x}}}
\def\rvy{{\mathbf{y}}}

\def\vtheta{{\bm{\theta}}}

\def\ve{{\bm{e}}}
\def\vf{{\bm{f}}}

\DeclareMathAlphabet{\mathsfit}{\encodingdefault}{\sfdefault}{m}{sl}
\SetMathAlphabet{\mathsfit}{bold}{\encodingdefault}{\sfdefault}{bx}{n}

\newcommand{\ptarget}{p_{\rm{target}}}
\newcommand{\pprior}{p_{\rm{prior}}}

\newcommand{\R}{\mathbb{R}}

\usepackage{microtype}
\usepackage{graphicx}
\usepackage{subfigure}
\usepackage{booktabs} 

\usepackage{enumitem}
\usepackage{multirow}   
\usepackage{algorithm}
\usepackage{algorithmic}

\usepackage{hyperref}
\usepackage{adjustbox}
\usepackage{makecell}

\usepackage[accepted]{icml2025}

\usepackage{amsmath}
\usepackage{amssymb}
\usepackage{mathtools}
\usepackage{amsthm}

\usepackage[capitalize,noabbrev]{cleveref}

\usepackage{colortbl}
\usepackage{xcolor}
\theoremstyle{plain}
\newtheorem{theorem}{Theorem}[section]

\theoremstyle{definition}

\theoremstyle{remark}

\usepackage[textsize=tiny]{todonotes}

\icmltitlerunning{Single-Step Consistent Diffusion Samplers}

\begin{document}

\twocolumn[
\icmltitle{Single-Step Consistent Diffusion Samplers}



\icmlsetsymbol{equal}{*}

\begin{icmlauthorlist}
\icmlauthor{Pascal Jutras-Dubé}{yyy}
\icmlauthor{Patrick Pynadath}{yyy}
\icmlauthor{Ruqi Zhang}{yyy}
\end{icmlauthorlist}

\icmlaffiliation{yyy}{Department of Computer Science, Purdue University, West Lafayette, USA}

\icmlcorrespondingauthor{Pascal Jutras-Dubé}{pjutrasd@purdue.edu}
\icmlcorrespondingauthor{Ruqi Zhang}{ruqiz@purdue.edu}


\vskip 0.3in
]



\printAffiliationsAndNotice{}  

\begin{abstract}
Sampling from unnormalized target distributions is a fundamental yet challenging task in machine learning and statistics.
Existing sampling algorithms typically require many iterative steps to produce high-quality samples, leading to high computational costs that limit their practicality in time-sensitive or resource-constrained settings.
In this work, we introduce \emph{consistent diffusion samplers}, a new class of samplers designed to generate high-fidelity samples in a single step.
We first develop a distillation algorithm to train a consistent diffusion sampler from a pretrained diffusion model without pre-collecting large datasets of samples.
Our algorithm leverages incomplete sampling trajectories and noisy intermediate states directly from the diffusion process.
We further propose a method to train a consistent diffusion sampler from scratch, fully amortizing exploration by training a single model that both performs diffusion sampling and skips intermediate steps using a self-consistency loss.
Through extensive experiments on a variety of unnormalized distributions, we show that our approach yields high-fidelity samples using less than 1\% of the network evaluations required by traditional diffusion samplers.
\end{abstract}

\section{Introduction}\label{sec:introduction}
Sampling from densities of the form 
\begin{equation}
    \ptarget = \frac \rho Z, \quad \text{with} \quad Z = \int_{\mathbb R^d} \rho(\rvx)\deriv \rvx 
\end{equation}
with $\rho$ evaluable pointwise but $Z$ intractable, is a central problem in machine learning \citep{neal1995bayesian, hernandez2015probabilistic} and statistics \citep{neal2001annealed, andrieu2003mcmc}, and has applications in scientific fields like physics \citep{wu2019solving, albergo2019flow, No2019BoltzmannGS}, chemistry \citep{frenkel2002molecularsimulation, hollingsworth2018moleculardynamics, holdijk2024ocmolecule}, and many other fields involving probabilistic models. 

Many established sampling algorithms are inherently iterative, with the accuracy of the final samples depending heavily on the number of steps. 
Classical Markov chain Monte Carlo (MCMC) methods asymptotically converge to the target distribution as the number of steps goes to infinity\citep{mackay2003mcmcbook, robert1995convergencemcmc}, while more recent diffusion-based approaches \citep{zhang2022pis, vargas2023dds, berner2024dis} guarantee convergence in a finite number of steps but often necessitate hundreds of iterations to yield high-quality samples. 
Such iterative samplers tend to suffer from slow mixing, making them impractical for use in large models and resource-limited scenarios.

Recent work on diffusion generative models \citep{sohl2015thermo, ho2020ddpm, song2019smld, song2021sde} have proposed fewer-step sampling via more efficient differential equation solvers \citep{song2021ddim, jolicoeurmartineau2021gotta, karras2022edm} or knowledge distillation \citep{salimans2022progdist, song2023consistency}, which enables single-step generation.
However, directly applying these distillation techniques to unnormalized distributions is challenging, as it often requires large datasets of samples that may be expensive to collect.
This motivates the following question:

\emph{Can we significantly reduce the steps required by samplers, enabling few-step or even single-step sampling?}

In this paper, we propose \emph{consistent diffusion samplers} to produce high-quality samples in a single step.
We first show that diffusion-based samplers can be \emph{consistently distilled} into single-step diffusion samplers. 
Instead of storing a large dataset of fully diffused samples, our approach exploits incomplete trajectories and noisy samples encountered during the diffusion process.
We further introduce a \emph{self-consistent} diffusion sampler that does not require a pretrained diffusion sampler.
Instead, it fully amortizes exploration by jointly learning both diffusion sampling and large cut off steps that match the outcome of paths of small steps.
This enables single-step sampling yet retains the option to refine samples through multiple iterations if desired, subsuming existing diffusion-based approaches.

Our contributions can be summarized as follows:
\begin{itemize}[itemsep=2pt, topsep=0pt]
\item We show that diffusion-based samplers for unnormalized distributions can be effectively distilled into single-step consistent samplers without pre-collecting large datasets of samples.
\item We introduce a self-consistent diffusion sampler that learns to perform single-step sampling by jointly training diffusion-based transitions and large shortcut steps via a self-consistency criterion. 
This method only trains one neural network and does not require pretrained samplers or high-quality data.
\item Through extensive evaluations on synthetic and real unnormalized distributions, we demonstrate that our method delivers competitive sample quality while drastically reducing sampling steps.
\end{itemize}

\section{Related Work}\label{sec:related-work}
\paragraph{Markov chain Monte Carlo (MCMC)}
Markov chain Monte Carlo methods are a classical approach for sampling from unnormalized target densities. 
The key idea is to construct a Markov chain whose stationary distribution matches the target distribution \citep{brooks2012handbookmcmc}. 
Prominent examples include the Metropolis-Hastings algorithm \citep{metropolis1953equation, hastings1970monte}, Gibbs sampling \citep{geman1984gibbs}, and Langevin dynamics \citep{rossky1978langevin, parisi1981langevin}. 
By exploiting geometric structure in the target distribution, Hamiltonian Monte Carlo \citep{duance1987hmc, mackay2003mcmcbook, brooks2012handbookmcmc, chen2014stochastic} often leads to more efficient exploration. 
To address scalability challenges in high-dimensional or large-dataset scenarios, stochastic gradient MCMC variants \citep{welling2011langevin, chen2014stochastic, zhang2020amagold, zhang2019cyclical} have been introduced. 
Although these MCMC methods reduce per-step computational costs or improve mixing, they remain inherently iterative, requiring many transitions to yield high-quality samples.

\paragraph{Learning-Based Samplers}
Amortized inference shifts the computational overhead from test-time sampling to a training phase, allowing for faster inference \citep{gershman2014amortized}. 
Approaches such as amortized MCMC~\citep{li2017amortizedmcmc} train a neural network to mimic the distribution of samples obtained after $T$ transitions of a traditional MCMC process. 
Similarly, GFlowNets \citep{bengio2021gflownets, bengio2023foundations} learn to sequentially construct complex discrete objects, effectively learning a sampling strategy. 
While GFlowNets amortize the computational challenges of lengthy stochastic searches and mode-mixing
during training, their sampling process remains sequential, as objects are constructed step-by-step
through a series of constructive steps.

An alternative viewpoint casts the sampling problem as an optimal control task \citep{zhang2022pis, berner2024dis, richter2024improved}, where one trains a controlled stochastic differential equation to transport an initial distribution to the target via a Schrödinger bridge \citep{schrodinger1931umkehrung, schrodinger1932relativiste}. 
This perspective motivates recent efforts to use diffusion-based samplers \citep{geffner2023langevin, vargas2023dds, zhang2024dgfs, phillips2024particle, chen2025SCLD}. 
While such diffusion and flow-based frameworks have advanced the state of the art, they require numerical solvers operating on dense time discretizations.

\paragraph{Consistent Generative Models}
Recent work in generative modeling has explored the concept of consistency: ensuring that large transitions between observed distributions are consistent with sequences of incremental transformations. 
Consistency models \citep{song2023consistency, song2023improved, lu2025simplifying} 
learn a direct mapping from any point in time to the terminal state. 
Progressive distillation \citep{salimans2022progdist, meng2023distillation} incrementally distills a trained diffusion model into a more efficient version that takes half as many until a single-step model is achieved.
Similarly, shortcut models \citep{liu2023flowstraight, frans2025shortcut} leverage progressive self-distillation during training to achieve accelerated inference without relying on a pre-trained teacher model.

These methods focus on generative modeling tasks and assume access to a dataset drawn from the target distribution.
Our work introduces the notion of consistency into the setting of sampling from unnormalized densities. 
We assume access only to an unnormalized pointwise oracle $\rho$ for the target density, without requiring any pre-collected samples.

\section{Preliminaries: Diffusion-Based Samplers}\label{sec:background}
Diffusion-based samplers are controlled stochastic differential equations (SDEs) that transport samples from a simple prior distribution $\pprior$ to the target distribution $\ptarget$.
Consider a forward-time SDE over $t \in [0,T]$ with initial condition $\rvx_0 \sim \pprior$:
\begin{equation}\label{eq:generative-sde} 
     \deriv \rvx_t = \bigl(\mu(t)\rvx_t + g(t)u_\theta(\rvx_t, t))\bigr) \deriv t + g(t)  \deriv \rvw_t,
\end{equation} 
where $\rvw$ is a standard Brownian motion, $\mu$ is the drift term, $g$ is the diffusion coefficient, and $u_\theta$ is a learned control term parameterized by a neural network.

Further consider the time-reversal process $\rvy$ of a diffusion that gradually adds noise to samples from the target distribution:
\begin{equation}\label{eq:target-sde}
\deriv \rvy_t = \bigl(\mu(t)\rvy_t + g^2(t)\nabla\log p_{\rvy_t}(\rvy_t)\bigr)\deriv t + g(t)\deriv \rvw_t.
\end{equation}
If we choose $\rvy_0 \sim \pprior$ and $\mu$ and $g$ such that $\rvy_T \sim \ptarget$, then setting $u_\theta(\rvx_t, t) = g(t)\nabla\log p_{\rvy_t}(\rvx_t)$ in Eq.~\ref{eq:generative-sde} would yield $p_{\rvx_t}=p_{\rvy_t}$ and thus $\rvx_T \sim \ptarget$ \citep{Anderson1982ReversetimeDE}.
In practice, however, the score function $\nabla\log p_{\rvy_t}$ is unknown and must be approximated by training $u_\theta$.

Let $\mathbb{P}_{\rvx}$ denote the path space measure induced by the SDE in Eq.~\ref{eq:generative-sde}, and $\mathbb{P}_{\rvy}$ the path space measure for the time-reversed process in Eq.~\ref{eq:target-sde}. Further, let $\mathcal{U}\subset C\bigl(\R^d\times[0,T],\R^d\bigr)$ be a space of admissible controls. 
From an optimal control and path space perspective \citep{berner2024dis, richter2024improved}, the diffusion sampling problem can be framed as finding an optimal control $u^*$ that minimizes a divergence between these two path measures:
\begin{equation}\label{eq:sampling-problem}
u^* \in \underset{\mathcal{U}}{\arg\min}\; D(\mathbb{P}_{\rvx} \,\|\, \mathbb{P}_{\rvy}),
\end{equation}
where $D(\cdot\,\|\,\cdot)$ is an appropriate divergence.  

To evaluate $D(\mathbb{P}_{\rvx} \,\|\, \mathbb{P}_{\rvy})$, one requires the Radon–Nikodym derivative, which measures how much more likely a given trajectory $\rvv$ is under $\mathbb{P}_{\rvx}$ than under $\mathbb{P}_{\rvy}$:
\begin{equation}\label{eq:RN}
    \frac{\deriv\mathbb{P}_\rvx}{\deriv\mathbb{P}_{\rvy}}(\rvv) = Z \exp\bigl( R(\rvv) + S(\rvv) + B(\rvv)\bigr)
\end{equation}
where
\begin{align*}
    &R(\rvx) = \int_0^T \left(\tfrac12\|u_\theta(\rvx_t, t)\|^2 - \operatorname{div}(\mu(t)\rvx_t)\right)\deriv t,\\
    &S(\rvx) = \int_0^T u_\theta(\rvx_t, t) \deriv \rvw_t, \quad \text{and}\\
    &B(\rvx) = \log\frac{\pprior(\rvx_0)}{\rho(\rvx_T)}.
\end{align*}

Two widely used divergences in diffusion-based sampling are:
\begin{align}
    &D_\text{KL}(\mathbb{P}_{\rvx} \,\|\, \mathbb{P}_{\rvy}) = \mathbb E \bigl[R(\rvx) + B(\rvx)\bigr] + \log Z \label{eq:kl};\\
    &D_\text{LV}(\mathbb{P}_{\rvx} \,\|\, \mathbb{P}_{\rvy}) = \mathbb{V}\bigl[R(\rvx) + S(\rvx) + B(\rvx)\bigr].\label{eq:lv} 
\end{align} 

Here, $D_{\mathrm{KL}}$ is the Kullback–Leibler divergence \citep{zhang2022pis,vargas2023dds,berner2024dis}, and $D_{\mathrm{LV}}$ is the log-variance divergence \citep{richter2024improved}. 

Once trained, the control $u_\theta$ allows for generating samples from $\ptarget$ by simulating the forward SDE in Eq.~\ref{eq:generative-sde}.
In practice, numerical discretization $0 = t_1 < t_2 < \ldots < t_N = T$ is required, and finer time steps yield more accurate sampling but at higher computational cost. 
Thus, a key challenge lies in balancing step size against the desired accuracy and efficiency.

\section{Consistency Distilled Diffusion Samplers}\label{sec:CDDS}
In this section, we show how to adapt consistency distillation to the problem of sampling from unnormalized densities.
We name our method the \emph{consistency distilled diffusion sampler} (CDDS).
The next section will address how to remove the requirement of having a pre-trained diffusion sampler.

Our goal is to learn a consistency function $f: (\rvx_t, t) \mapsto \rvx_T$, which maps any intermediate state $\rvx_t$ directly to a sample $\rvx_T$ from the target distribution. 
Although we lack a dataset of samples from $\ptarget$, if we possess a pre-trained diffusion sampler, we can approximate such a dataset by simulating the generative SDE in Eq.~\ref{eq:generative-sde}, producing samples $\{\hat{\rvx}_T^i\}_{i=1}^M$. 
We can then apply either consistency distillation or consistency training (as in Algorithms 2 and 3 of \citealp{song2023consistency}) to learn $f$.
This approach is expensive as it necessitates pre-collecting and storing a large dataset.

Consider a pre-trained diffusion process whose trajectories
$\rvx_{t_1}, \rvx_{t_2}, \ldots, \rvx_{T}$ would normally be used to create a dataset for distillation.
Instead, we directly leverage intermediate states $\rvx_t$ during each training iteration.
This reduces storage demands and limits the accumulation of numerical errors that could arise from fully integrating the numerical solver.
If the error per step of an order-$p$ solver is bounded by $O((t_{n+1} - t_n)^{p+1})$, using multiple, shorter intervals can help keep the overall global error smaller.

One challenge in using intermediate states from a stochastic diffusion is the inherent randomness of the SDE trajectory, which complicates the mapping $(\rvx_t, t) \mapsto \rvx_T$. 
To address this, we simulate the associated probability flow (PF) ODE \citep{song2021sde}:
\begin{equation}\label{eq:pf-ode-forward}
\odif\rvx_t
= \Bigl(\mu(\rvx_t, t) + \tfrac12 \sigma(t),u(\rvx_t, t)\Bigr)\odif t,
\end{equation}
which shares the same marginal distributions as the original SDE but follows a deterministic trajectory. 
Integrating the PF ODE at discrete times $t_n$ and $t_{n+1}$ gives intermediate points $\hat{\rvx}_{t_n}$ and $\hat{\rvx}_{t_{n+1}}$, which we use for training.

We minimize the discrepancy between the outputs of the consistency function at consecutive intermediate states:
\begin{equation}\label{eq:cdloss}
\begin{aligned}
    \mathcal{L}_\text{CD} &(\vtheta, \vtheta^\prime; u)\\
&:= \mathbb{E}\Bigl[\lambda(t_n) d\bigl(f_{\vtheta^\prime}(\hat{\rvx}_{t_{n+1}}, t_{n+1}), f_{\vtheta}(\hat{\rvx}_{t_n}, t_n)\bigr)\Bigr],
\end{aligned}
\end{equation}
where $d(\cdot,\cdot)$ is a distance metric, $\lambda(\cdot)$ is a positive weighting function, and $\vtheta^{\prime} = \operatorname{stopgrad}(\vtheta)$ indicates that the gradients are not passed through the target term. 
Notably, different to training consistency generative models, here, both $\hat{\rvx}_{t_{n+1}}$ and $\hat{\rvx}_{t_n}$ are approximate states obtained by partially integrating the PF ODE. 
Training a consistent diffusion sampler via distillation requires a similar computational cost as training the original diffusion sampler, since both processes involve simulating trajectories; however, it enables faster inference at test time.
The training procedure is summarized in Algorithm~\ref{alg:cd} and illustrated in Figure~\ref{fig:cd}. 

\begin{algorithm}[tb]
\caption{Data-Free Consistency Distillation}\label{alg:cd}
\begin{algorithmic}
\STATE \textbf{Input:} model parameters $\vtheta$, control $u$, learning rate $\eta$, distance $d$, weight $\lambda$
\STATE $\vtheta^{\prime} \leftarrow \vtheta$
\REPEAT
\STATE Sample $\rvx_0 \sim \pprior$ and $n \sim \mathcal{U}\{1, N-1\}$
\STATE Integrate Eq.~(\ref{eq:pf-ode-forward}) to obtain $\hat{\rvx}_{t_n}$ and $\hat{\rvx}_{t_{n+1}}$
\STATE $\mathcal{L}(\vtheta,\vtheta^{\prime}; u) \leftarrow \lambda(t_n)d\bigl(f_{\vtheta^{\prime}}(\hat{\rvx}_{t_{n+1}}, t_{n+1}),f_{\vtheta}(\hat{\rvx}_{t_n}, t_n)\bigr)$
\STATE $\vtheta \leftarrow \vtheta - \eta\, \nabla\vtheta \mathcal{L}(\vtheta,\vtheta^{\prime}; u)$
\STATE $\vtheta^{\prime} \leftarrow \text{stopgrad}(\vtheta)$
\UNTIL{convergence}
\end{algorithmic}
\end{algorithm}

\begin{figure}[bt!]
\centering
\includegraphics[width=0.9\linewidth]{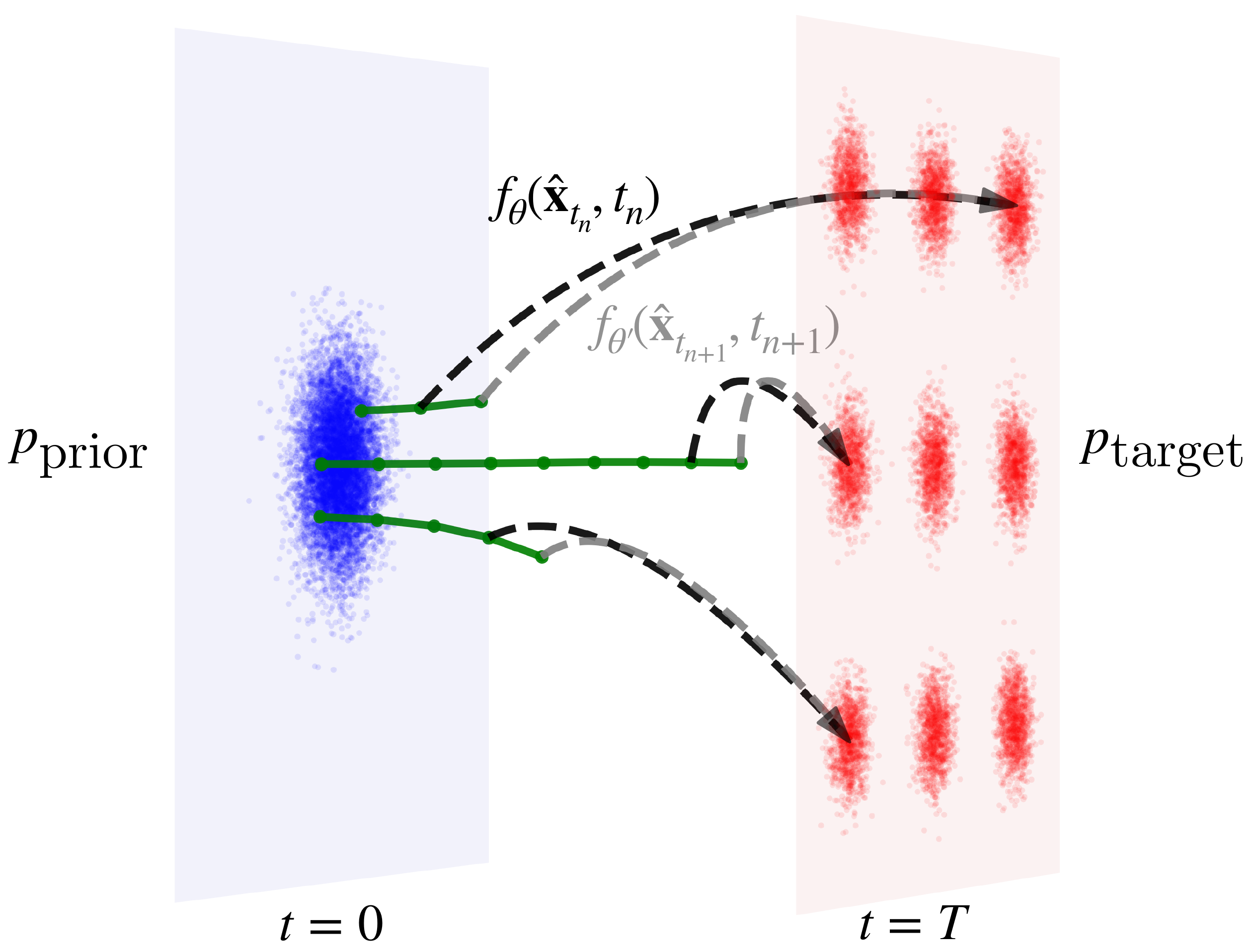}
\caption{Consistency distilled diffusion samplers learn to map consecutive intermediate states (black and gray dots) along partial ODE trajectories (green curve) directly to the terminal state.}
\label{fig:cd}
\end{figure}

If the loss in Eq.~\ref{eq:cdloss} is driven to zero, the learned consistency function can approximate the true mapping arbitrarily well, provided the step size of the ODE solver is sufficiently small. 
We formally state this in Theorem \ref{thm:cm-bound}.
\begin{theorem}\label{thm:cm-bound}
    Let $\vf_\vtheta(\rvx_t, t)$ be a consistency function parameterized by $\vtheta$, and let $\vf(\rvx_t, t; u)$ denote the consistency function of the PF ODE defined by the control $u$. 
    Assume that $\vf_\vtheta$ satisfies a Lipschitz condition with constant $L > 0$, such that for all $t \in [0, T]$ and for all $\rvx_t, \rvy_t$,
    $$
        \| \vf_\vtheta(\rvx_t, t) - \vf_\vtheta(\rvy_t, t)\|_2 \leq L \|\rvx_t - \rvy_t\|_2.
    $$
    Additionally, assume that for each step $n \in \{1, 2, \ldots, N-1\}$, the ODE solver called at $t_{n}$ has a local error bounded by $O((t_{n+1} - t_n)^{p+1})$ for some $p \geq 1$.
    
    If, additionally, $\mathcal{L}_\text{CD}(\vtheta, \vtheta; u) = 0$, then:
    $$
    \sup_{n,\rvx_{t_n}} \|\vf_\vtheta(\rvx_{t_n}, t_n) -  \vf(\rvx_{t_n}, t_n; u)\|_2 = O((\Delta t)^p),
    $$
    where $\Delta t := \max_{n \in \{1, 2, \ldots, N-1\}} |t_{n+1} - t_n|$. 
\end{theorem}
A complete proof is provided in Appendix \ref{sec:a-proof}.

While our distillation approach builds upon the core principles of consistency models, it differs in setting and requirements.
Consistency generative models assume direct access to real samples from the target distribution. 
In contrast, our consistency distilled diffusion samplers address the problem of sampling from unnormalized target densities, where no dataset of target samples is available. 
Our method extends consistency distillation to sampling from unnormalized distributions, making it applicable beyond generative modeling tasks.

\section{Self-Consistent Diffusion Samplers}\label{sec:SCDS}
In this section, we introduce \emph{self-consistent diffusion sampler} (SCDS) that achieves single-step sampling without requiring a pre-trained diffusion sampler. 
Our motivation stems from merging two complementary perspectives.

First, diffusion-based samplers learn a time-dependent control function that steers an SDE from a simple prior distribution to the target distribution. 
Typically, the control is trained on a fixed schedule (e.g., $N$ small increments of length $T/N$ along a discretized time axis), requiring multiple steps.
Second, consistency models learn a direct mapping from any intermediate state on an ODE to the terminal state. 
In other words, at time $t$ the model is implicitly taught to jump a large step of length $T - t$.

Our idea is to unify these approaches in a single model. 
Specifically, we condition a control function $u_\theta(\rvx_t, t, d)$ on both the current time $t$ and the desired step size $d$. 
By adjusting $d$, the model can adapt between short incremental steps (as in standard diffusion samplers) and large jumps (as in consistency models). 
This design amortizes the learning of both small and large transitions into one network and recovers consistency models' single-step sampling by setting $d = T - t$ and diffusion sampling by setting $d = T/N$.
In doing so, we avoid training two separate models. 

\paragraph{Enforcing Self-Consistency}
To ensure that the step-size-conditioned control function $u_\theta(\rvx_t, t, d)$ remains accurate across varying step sizes, we introduce a self-consistency loss. 
The key idea is that taking a large step should yield the same result as taking multiple smaller steps.
To do so, we impose a consistency condition on the Euler discretization of the PF ODE in Eq.~\ref{eq:pf-ode-forward}.
Specifically, a single large step of size $2d$,
\begin{equation}\label{eq:shortcut}
\rvx_{t+2d}= \rvx_t + \left( \mu(t)\rvx_t + \tfrac12 g(t)u_\theta(\rvx_t, t, 2d)\right) 2d,
\end{equation}
must equal two smaller steps of size $d$.
The intermediate state is computed as
\begin{equation*}\label{eq:intermediate}
    \rvx_{t+d}^{\prime}  = \rvx_t + \left( \mu(t)\rvx_t + \tfrac12 g(t) u_{\vtheta'}(\rvx_t, t, d) \right) d
\end{equation*}
and the final state after two steps is
\begin{equation}\label{eq:twostep}
\begin{aligned}
    &\rvx_{t+2d}^{\prime}  = \rvx_{t+d}^{\prime}  \\
    &+ \left( \mu(t+d)\rvx_{t+d} + \tfrac12 g(t+d) u_{\vtheta'}(\rvx_{t+d}, t+d, d) \right) d,
\end{aligned}
\end{equation}
where $\vtheta^{\prime} = \operatorname{stopgrad}(\vtheta)$. 
The self-consistency objective is a simple least square minimization problem:
\begin{equation}\label{eq:sc-loss}
    \mathcal L_{\text{SC}} = \mathbb E\left[\left\lVert \rvx_{t+2d}^{\prime} - \rvx_{t+2d} \right\rVert^2\right]
\end{equation}
where the expectation is taken over time indices and step sizes drawn from the simulated trajectories.

This loss encourages the model to correct for numerical errors when taking large steps, allowing it to “skip” multiple smaller steps while remaining consistent with the dynamics of the PF ODE.
To initiate this recursive training, we must define and learn the behavior at the base case $d=T/N$.
\begin{figure*}[bt!]
\centering
\includegraphics[width=1.0\linewidth]{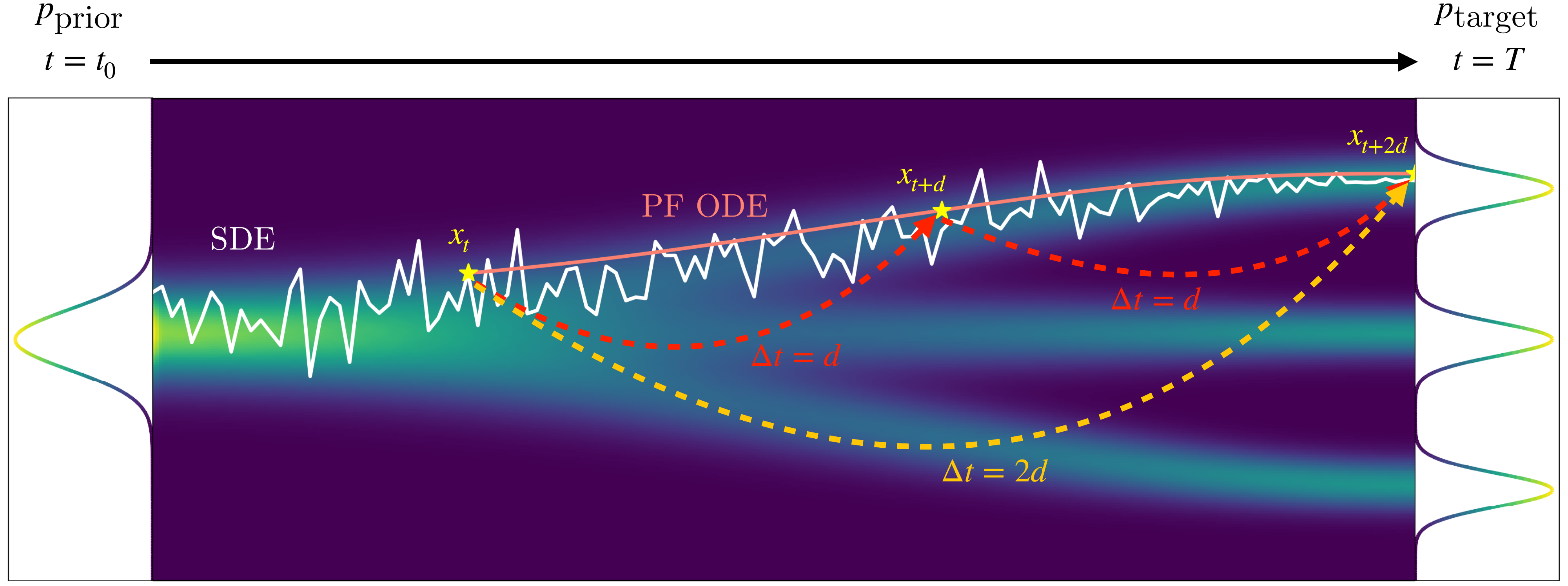}
\caption{Graphical illustration of the training procedure for SCDS over the path space.
First, the SDE trajectory (white) is simulated to compute the sampling loss $\mathcal L_{S}$.
Next, a timestep $t$ and a step size $d$ are randomly sampled.
From $\rvx_t$ on the simulated SDE trajectory, we execute two consecutive steps of size $d$ (red) along the PF-ODE trajectory (pink), obtaining the target $\rvx_{t+2d}^{\prime}$.
Finally, the shortcut step of size $d$ (orange) predicts $\rvx_{t+2d}$ directly from $\rvx_t$, and the self-consistency loss $\mathcal L_{SC}$ minimizes the squared difference between $\rvx_{t+2d}$ and the two-step target $\rvx_{t+2d}^{\prime}$, ensuring multi-scale consistency.}
\label{fig:scds}
\end{figure*}

\paragraph{Learning the Base Case $\mathbf{d=T/N}$}
In standard generative modeling scenarios (where a dataset is available), the base case $d=T/N$ can be learned directly from data using deterministic trajectories \citep{lipman2023flowmatching, frans2025shortcut}. 
These trajectories provide explicit guidance toward high-density regions of the target distribution.

However, when working with an unnormalized density, the key challenge is discovering high-probability regions~(modes). 
In such cases, exploration is necessary to locate and model these regions effectively \citep{chen2025SCLD}. 
Diffusion-based samplers facilitate exploration through their stochastic dynamics: Brownian motion helps probe different parts of the space, allowing the model to learn and adapt itself to the target distribution.

Thus, diffusion-based sampling is particularly well-suited for learning the base case. 
The sampling objectives in Eq.~\ref{eq:kl} and Eq.~\ref{eq:lv} train the model by simulating the stochastic process in Eq.~\ref{eq:generative-sde}, allowing it to learn the structure of high-density regions. 
In this work, we adopt the log-variance divergence as our base sampling objective:
\begin{equation}\label{eq:sampling-loss}
    \mathcal{L}_\text{S} = D_\text{LV}(\mathbb{P}_{\rvx} \,\|\, \mathbb{P}_{\rvy}).
\end{equation}
By optimizing $u_\theta(\rvx_t, t, d=T/N)$ under this loss, we ensure that the model can generate meaningful transitions from the prior to these regions of interest, forming a strong foundation for self-consistent learning at larger step sizes.

\paragraph{End-to-End Training Algorithm}
Our training procedure jointly optimizes two objectives: (1) the sampling loss Eq.~\ref{eq:sampling-loss} for the base case  $d=T/N$, which ensures exploration and score approximation by simulating the SDE in Eq.~\ref{eq:generative-sde}, and (2) the self-consistency loss in Eq.~\ref{eq:sc-loss} enforced on the PF-ODE in Eq.~\ref{eq:pf-ode-forward} for larger $d$, which enforces consistency across multiple time scales.

To enable the recursive halving of steps, we discretize the time interval $[0,T]$ into $N+1$ points, where $N$ is chosen as a power of two.
The sampling loss is computed by simulating the forward SDE along this time grid.

For self-consistency training, we sample step sizes $d$ and times $t$ such that $d$ are powers of two (multiplied by $T/N$) dividing the remaining time $T - t$.
This ensures that from any time $t$, we can take exactly $k$ steps of size $d$ to reach the terminal state for some integer $k$. 
This way, training focuses on time sequences that are applicable during inference.

To compute the self-consistency loss, we extract $\rvx_t$ from the simulated forward SDE. 
Using $\rvx_t$ and the sampled step size $d$, we compute the shortcut step $\rvx_{t+2d}$ using Eq. \ref{eq:shortcut} and the two-step target trajectory $\rvx_{t+2d}^{\prime}$ using Eq. \ref{eq:twostep} on the PF ODE. 
We then optimize their squared difference via Eq.~\ref{eq:sc-loss}, ensuring that larger steps remain consistent with fine-grained trajectories.
The training procedure is summarized in Algorithm~\ref{alg:training-summary} and illustrated in Figure \ref{fig:scds}. 
\begin{algorithm}[tb]
\caption{SCDS Training}\label{alg:training-summary}
\begin{algorithmic}
\STATE \textbf{Input} Model parameters $\vtheta$, loss weightings $\lambda_\text{S}(\cdot)$ and $\lambda_\text{SC}(\cdot)$
\STATE $\vtheta' \leftarrow \vtheta$
\REPEAT
    \STATE Sample $\rvx_0 \sim \pprior$ and $(d, t) \sim p_{d,t}$.
    \STATE Compute $\rvx \leftarrow (\rvx_i)_{i=0}^T$ by simulating Eq.~\ref{eq:generative-sde} 
    \STATE Compute $\rvx_{t+2d}^{\prime}$ from Eq.~\ref{eq:shortcut}
    \STATE Compute $\rvx_{t+2d}$ from Eq.~\ref{eq:twostep}    
    \STATE Compute $\mathcal{L}_{\text{S}}$ using Eq.~\ref{eq:sampling-loss}
    \STATE Compute $\mathcal L_{\text{SC}}$ using Eq.~\ref{eq:sc-loss}.
    \STATE $\vtheta \leftarrow \nabla_\vtheta \left(\lambda_\text{S}(t)\mathcal{L}_{\text{S}} + \lambda_\text{SC}(t)\mathcal L_{\text{SC}}\right)$
    \STATE $\vtheta' \leftarrow \operatorname{stopgrad}{\vtheta}$
\UNTIL{convergence}
\end{algorithmic}
\end{algorithm}

Compared to previous diffusion-based samplers, our method only incurs $3$ additional network function evaluations per training iteration.

\paragraph{Few-step Sampling}
With a well-trained control $u_\theta$, sampling can be performed in a single step by drawing from the prior and applying a single Euler update with step size $d=T$, as shown in Algorithm~\ref{alg:single-step}. 
This accelerates generation compared to traditional diffusion-based samplers. 
Alternatively, our method provides a flexible tradeoff between computational efficiency and sample quality, allowing for multi-step refinement when needed, thus recovering standard diffusion-based sampling. 
This iterative procedure is detailed in Algorithm~\ref{alg:multi-step}.
\begin{algorithm}[tb]
\caption{Single-Step Sampling with SCDS}
\label{alg:single-step}
\begin{algorithmic}
\STATE \textbf{Input:} Trained model $u_\theta$
\STATE Sample $\rvx_0 \sim \pprior$
\STATE Compute $\rvx_{T} = \rvx_0 + \left( \mu(0) \rvx_0 + \tfrac12 g(0) u_\theta(\rvx_0, 0, T) \right) T$
\STATE \textbf{Return} $\rvx_T$
\end{algorithmic}
\end{algorithm}
\begin{algorithm}[tb]
\caption{Multi-Step Sampling with SCDS}
\label{alg:multi-step}
\begin{algorithmic}
\STATE \textbf{Input:} Trained model $u_\theta$,  number of sampling steps $K$
\STATE Sample $\rvx_0 \sim \pprior$
\STATE Initialize $d \gets T/K$ and $t \gets 0$
\FOR{$k = 1, \dots, K$}
    \STATE Compute
        $\rvx_{t+d} = \rvx_t + \left( \mu(t) \rvx_t + \tfrac12 g(t) u_\theta(\rvx_t, t, d) \right) d$
    \STATE Update $t \gets t + d$
\ENDFOR
\STATE \textbf{Return} $\rvx_T$
\end{algorithmic}
\end{algorithm}

\paragraph{Approximating $\mathbf{Z}$.}
A benefit of SCDS is the ability to estimate the intractable normalizing constant $Z$. 
By leveraging the relationship established in the KL divergence objective (Eq.~\ref{eq:kl}), we can approximate $\log Z$. 
Specifically, when the optimal control $u^* = g(t)\nabla\log p_{\rvy_t}(\rvx_t)$ is attained, the KL divergence $D_\text{KL}(\mathbb{P}_{\rvx} \,\|\, \mathbb{P}_{\rvy})$ reaches zero.
This implies
$$
	- \log Z = \min_{u \in \mathcal{U}} \mathbb E \bigl[R(\rvx) + B(\rvx)\bigr].
$$
Unlike CDDS and consistency models, which focus on solely sample generation, SCDS leverages the control-based formulation to handle both sampling and the normalizing constant, making it applicable to a broader range of probabilistic tasks.

\paragraph{Learning Shortcuts Without Data}
SCDS shares conceptual similarities with progressive distillation \citep{salimans2022progdist} and shortcut models \citep{frans2025shortcut}, both of which enforce that a large time step transition should be consistent with two half-sized transitions. 
However, these methods rely on access to a dataset or to a pre-trained teacher model. 
In contrast, SCDS operates entirely without data, learning both the diffusion process and shortcut connections directly from an unnormalized density. 
This independence from a pre-trained model grants SCDS greater flexibility in choosing the prior distribution, SDE formulation, and time discretization, without being constrained by the design choices of a teacher model.

\begin{table*}[t]
    \centering
    \caption{Comparison of different methods in terms of Sinkhorn distances (lower is better). We present results on tasks where ground-truth samples are available for evaluation. ``NFE'' refers to the number of function evaluations.}
    \begin{tabular}{ll|ccccc}
        \toprule
        \multicolumn{2}{c|}{\textbf{Sinkhorn ($\downarrow$)}} & 
        \multicolumn{5}{c}{\textbf{Target Distribution}} \\ 
        \textit{Sampler} & \textit{NFE} & \textit{GMM (2d) }& \textit{Image (2d) }& \textit{Funnel (10d)} & \textit{MW54 (5d)} & \textit{MW52 (50d)} \\ 
        \midrule 
        \multirow{3}{*}{\makecell{SCDS \\ \textit{(Ours) }}}& 128 & 0.0204 & 0.0169 & 5.2569 & 0.1191 & 7.4557 \\
        & 2 & 0.0279 & 0.0294 & 5.3488 & 0.1955 & 11.5200 \\ 
           & 1 & 0.0330 & 0.0322 & 5.3729 & 0.2102 & 7.4925\\ \midrule
        \multirow{2}{*}{\makecell{CDDS \\ \textit{(Ours)}}} & 2 & 0.0241 & 0.0309 & 7.1329 & 0.1570 & 6.5010 \\ 
            & 1 & 0.0224 & 0.0309 & 7.2159 & 0.1569 & 6.5285 \\ 
        \midrule
        PIS        & 128 & 0.6656 & 0.9168 & 5.9956 & 0.1223 & 7.2955 \\
        DDS        & 128 & 0.0709 & 1.5818 & 6.0467 & 0.1190 & 7.2842 \\ 
        DIS        & 128 & 0.0203 & 0.0170 & 5.1578 & 0.1197 & 7.3668 \\
        DIS  & 1 & 0.0551 & 0.2781 & 10.4033 & 6.4679 & 31.7883 \\ 
        \bottomrule
    \end{tabular}
    \label{tab:sinkhorn}
\end{table*}

\section{Experiments}\label{sec:experiments}
\paragraph{Experimental Setup.}
We evaluate our CDDS and SCDS on multiple sampling benchmarks: a 9-mode Gaussian mixture model in 2d (GMM), a 2d image of a labrador (Image), a 10d Funnel distribution, and two 32-mode many-well tasks (MW54 in 5d and MW52 in 50d). 
We also consider a high-dimensional log Gaussian Cox Process (LGCP) problem in 1600d.

We compare to three seminal diffusion samplers: path integral sampler (PIS) \citep{zhang2022pis}, denoising diffusion sampler (DDS) \citep{vargas2023dds}, and time-reversed diffusion sampler (DIS) \citep{berner2024dis}. 
We also show a single-step version of DIS as a naive baseline, primarily to gauge how single-step sampling might upper-bound the Sinkhorn distance if we remove any learned shortcut. 
In our experiments, CDDS is a distilled version of DIS, and is initialized from DIS weights. 
Similarily, the sampling loss in SCDS is computed as in DIS. 
We use Fourier features network to condition on the stepsize $d$ \citep{tancik2020fourier}.

When ground-truth samples are available, we measure performance via the Sinkhorn distance \citep{cuturi2013sinkhorn} between generated samples and samples from the target distribution. 
For the LGCP task, we report the relative error of the estimated normalizing constant $\log Z$. 
Additionally, we quantify the number of function evaluations (NFE) \citep{karras2022edm}, which corresponds to the total SDE/ODE discretization steps required for each sampler. 
For more details on the training of the various samplers, along with evaluation details and target distribution settings, see Appendix~\ref{sec:details}. 

\begin{figure}[tb!]
\centering
\includegraphics[width=1.0\linewidth]{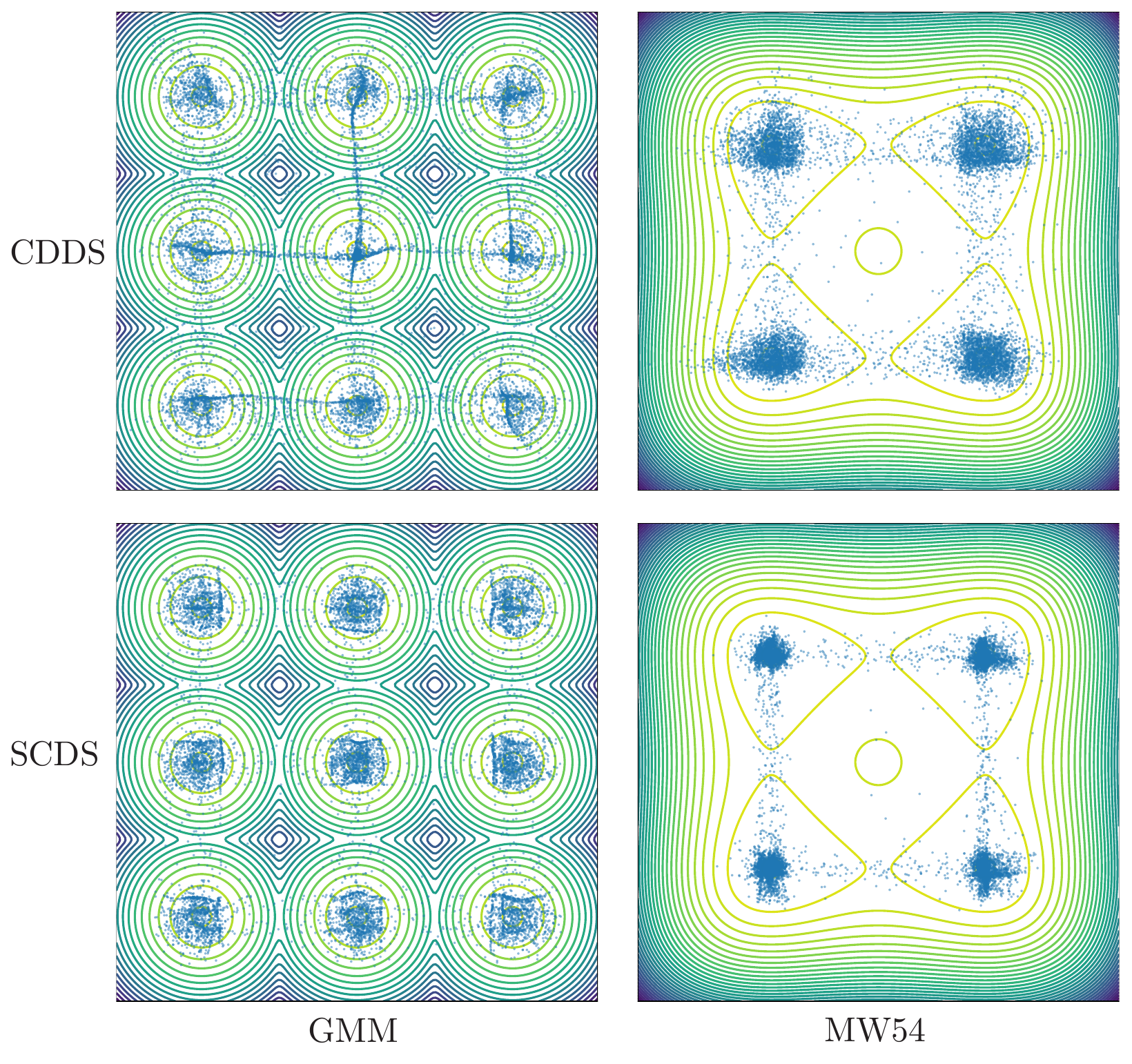}
\caption{Visualization of the GMM and MW54 tasks. 
CDDS and SCDS recover all modes in just a single sampling step.}
\label{fig:contour}
\end{figure}

\paragraph{Sinkhorn Results and Analysis.}
Table~\ref{tab:sinkhorn} shows that both CDDS and SCDS maintain competitive sinkhorn distances in single- and two-steps generations compared to existing diffusion-based samplers with $128$ steps. 
A single-step version of DIS is also listed in Table~\ref{tab:sinkhorn} to illustrate a naive upper bound on the distance. 
As expected, skipping all intermediate steps hurts sampling quality significantly. 
However,  even with only one step, SCDS and CDDS consistently outperform single-step DIS by a clear margin on every task, highlighting the benefits of enforcing consistency.
Figure~\ref{fig:contour} shows that CDDS and SCDS recover all modes when sampling using a single step on the GMM and MW54 tasks.

As with other consistency-based methods \cite{song2023consistency} we find CDDS’s multi-step performance typically saturates after 2--3 steps, indicating minimal gains from iterative refinements.
In contrast, SCDS’s accuracy steadily improves with increasing step counts in most tasks (see Figure~\ref{fig:histogram}), except for minor dips at 4 steps in Funnel and at 2/4 steps in MW52. 
Such dips may arise from partial coverage challenges or local minima in training when bridging intermediate steps in relatively high dimensional data; nonetheless, the general upward trend demonstrates that SCDS effectively recovers standard multi-step diffusion behaviors. 
Moreover, SCDS often compares to or surpass PIS and DDS at 128 steps, thanks to the log-variance objective and the optimal control perspective from \citet{berner2024dis, richter2024improved}.
Interestingly, on the 50-dimensional MW52 task, CDDS attains a lower Sinkhorn distance than all baselines. 
We hypothesize that distillation, by leveraging the PF ODE of a well-trained DIS, learns smoother transitions that are especially beneficial in high-dimensional settings.
\begin{figure*}[tb!]
\centering
\includegraphics[width=1.0\linewidth]{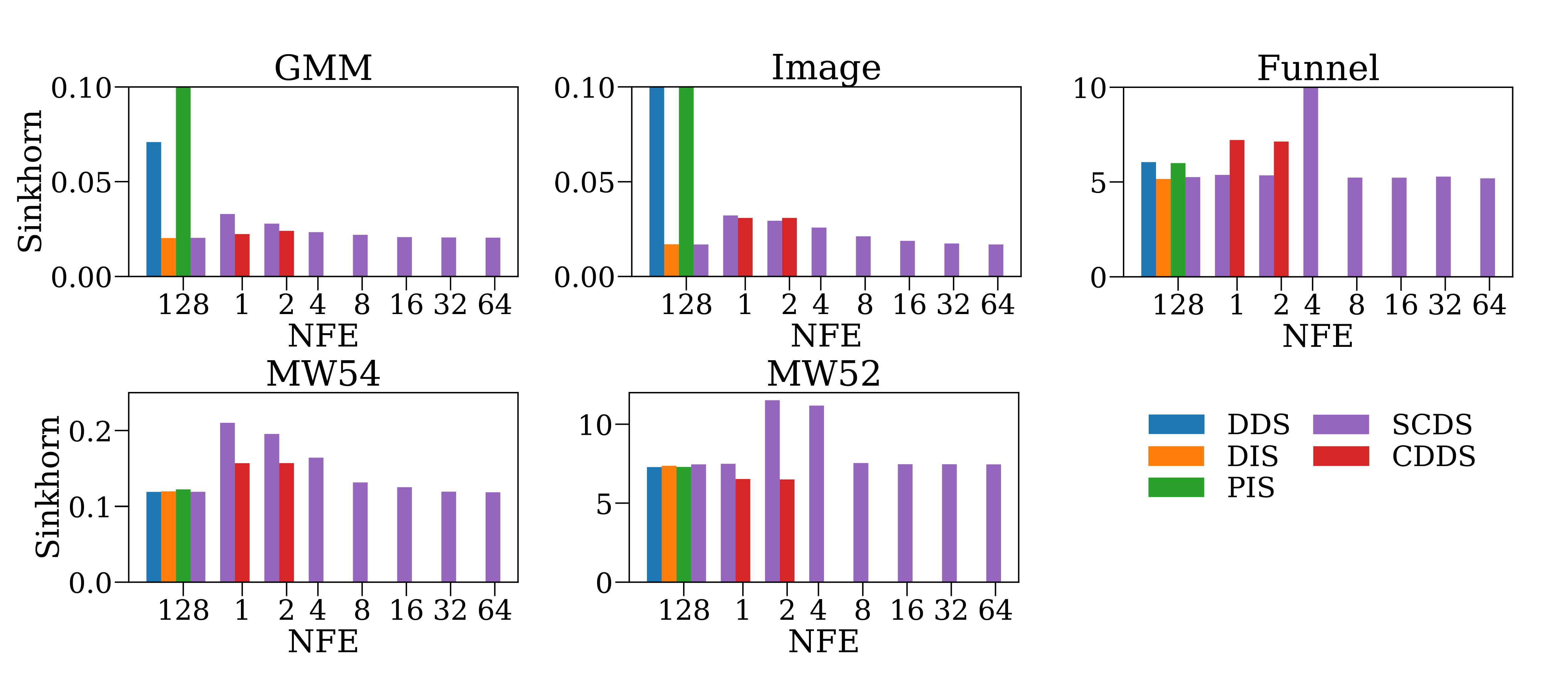}
\caption{Comparison of Sinkhorn distance for a range of NFEs between the proposed consistency samplers (CDDS, SCDS) and diffusion-based samplers (PIS, DDS, DIS). 
For most targets, CDDS and SCDS show competitive Sinkhorn values with baselines with much lower NFEs.}
\label{fig:histogram}
\end{figure*}

\paragraph{Log Gaussian Cox Process.}
\begin{table}[t]
    \centering
    \caption{Relative error of Log $Z$ estimates for various samplers on LGCP target distribution.}
    \begin{tabular}{ll|c}
        \toprule
        \multicolumn{3}{c}{\textbf{LGCP (1600d)}}\\ 
        \textit{Sampler}  & \textit{NFE} & \textit{Log} $Z$ \textit{Error} ($\downarrow$)\\ 
        \midrule
        \multirow{8}{*}{\makecell{SCDS \\
        \textit{(Ours)}}} & 128 & 0.9968 \\ 
         & 64 & 1.0506  \\ 
        & 32 & 1.5976  \\ 
         & 16 & 2.2378  \\ 
         & 8 & 2.7931  \\ 
         & 4 & 3.9660 \\ 
          & 2 & 6.2420  \\ 
          & 1 & 9.9877 \\  \midrule
        PIS        & 128 & 0.2910 \\
        DDS        & 128 & 2.8545  \\ 
        \midrule
        \multirow{2}{*}{DIS}        & 128 & 0.3736  \\ 
          & 1 & 3094.7296  \\
        \bottomrule
    \end{tabular}
    \label{tab:logz}
    \vspace{-1em}
\end{table}
Table~\ref{tab:logz} compares $\log Z$ estimation errors for each method on the 1600d LGCP task. 
Multi-step PIS and DIS achieve smaller errors then SCDS, but SCDS remains viable even at reduced NFEs. 
Notably, as expected, single-step DIS fails catastrophically, whereas single-step SCDS remains stable. 

Since SCDS learns a time-dependent control function, it retains a connection to the Radon-Nikodym derivative in Eq.~\ref{eq:RN}, allowing for partition function estimation.
In contrast, CDDS (and consistency models in general) lack an explicit control representation, meaning they cannot directly estimate $Z$.
This is a key advantage of SCDS in applications where unnormalized densities must be integrated, such as Bayesian inference.

\vspace{-1em}
\paragraph{Discussion.}
Our methods target scenarios where reducing sampling complexity is critical. 
A key advantage of SCDS lies in its ability to learn both the diffusion sampling process and the self-consistency shortcuts \emph{simultaneously}. 
In contrast to consistency models, which require a pre-trained sampler or high-fidelity trajectories for distillation, SCDS forgoes such prerequisites and instead enforces consistency during training. 
This design choice is supported by our empirical results showing that SCDS is often competitive with well-established diffusion samplers and consistency-distilled approach CDDS that benefit from a carefully tuned, pre-trained teacher. 
Moreover, SCDS adapts seamlessly from single-step to many-step sampling without retraining, making it ideal for real-world applications with varying computational budgets or latency constraints.

\section{Conclusion}\label{sec:conclusion}
We introduced two novel approaches for efficient sampling from unnormalized target distributions: \emph{consistency-distilled diffusion samplers} (CDDS) and the \emph{self-consistent diffusion sampler} (SCDS). 
CDDS uses consistency distillation without generating a large dataset of samples. 
SCDS requires no pre-trained samplers and 
simultaneously learns to sample high-density regions and to take large steps across the path space. 
Our empirical results across a range of benchmarks demonstrate that both methods achieve competitive accuracy with as few as one or two steps. 
These findings highlight the potential of consistency-based methods for sampling from unnormalized densities. 

\bibliography{references}
\bibliographystyle{icml2025}

\newpage
\appendix
\onecolumn
\section{Consistency Distillation Proof}\label{sec:a-proof}
\renewcommand{\thetheorem}{4.\arabic{theorem}} 
\begin{theorem}
    Let $\vf_\vtheta(\rvx_t, t)$ be a consistency function parameterized by $\vtheta$, and let $\vf(\rvx_t, t; u)$ denote the consistency function of the PF ODE defined by the control $u$. 
    Assume that $\vf_\vtheta$ satisfies a Lipschitz condition with constant $L > 0$, such that for all $t \in [0, T]$ and for all $\rvx_t, \rvy_t$,
    $$
        \| \vf_\vtheta(\rvx_t, t) - \vf_\vtheta(\rvy_t, t)\|_2 \leq L \|\rvx_t - \rvy_t\|_2.
    $$
    Additionally, assume that for each step $n \in \{1, 2, \ldots, N-1\}$, the ODE solver called at $t_{n}$ has a local error bounded by $O((t_{n+1} - t_n)^{p+1})$ for some $p \geq 1$.
    
    If, additionally, $\mathcal{L}_\text{CD}(\vtheta, \vtheta; u) = 0$, then:
    $$
    \sup_{n,\rvx_{t_n}} \|\vf_\vtheta(\rvx_{t_n}, t_n) -  \vf(\rvx_{t_n}, t_n; u)\|_2 = O((\Delta t)^p),
    $$
    where $\Delta t := \max_{n \in \{1, 2, \ldots, N-1\}} |t_{n+1} - t_n|$.
\end{theorem}
\begin{proof}
    The proof is similar to the one presented by \citet{song2023consistency}, with the key difference that we must account for the global integration error introduced by the ODE solver.

    If the ODE solver, when called at $t_{n+1}$, has a local error uniformly bounded by $O((t_{n} - t_{n-1})^{p+1})$, then the cumulative error across all steps is approximately the sum of $n+1$ local errors and is bounded by $O((\Delta t)^p)$.

    We are interested in $\ve_n$, the error between the learned consistency function and the consistency function of the PF ODE defined by the control $u$ at $\rvx_{t_n} \sim p_{t_n}(\rvx_{t_n})$,
    $$
    \ve_n:=\vf_\vtheta(\rvx_{t_n}, t_n) - \vf(\rvx_{t_n}, t_n; u). 
    $$
        
    If $\mathcal{L}(\vtheta, \vtheta; u) = 0$, we deduce that
    $$
    \lambda(t_n) d(f_{\vtheta}(\hat{\rvx}_{t_{n+1}}, t_{n+1}), f_{\vtheta}(\hat{\rvx}_{t_{n}}, t_n)) = 0.
    $$
    Since $\lambda(t_n) > 0$, this implies:
    \begin{equation}\label{eq:learned-cf}
        \vf_\vtheta(\hat \rvx_{t_{n+1}}, t_{n+1}) = \vf_\vtheta(\hat \rvx_{t_{n}}, t_{n}).
    \end{equation}

    We can derive a recurrence relation for $\ve_n$:
    \begin{align*}
        \ve_n &\overset{(i)}{=}  \vf_\vtheta(\rvx_{t_n}, t_n) - \vf_\vtheta(\hat\rvx_{t_n}, t_n) + \vf_\vtheta(\hat\rvx_{t_n}, t_n) - \vf(\rvx_{t_{n+1}}, t_{n+1}; u)\\
        &\overset{(ii)}{=} \vf_\vtheta(\rvx_{t_n}, t_n) - \vf_\vtheta(\hat\rvx_{t_n}, t_n) + \vf_\vtheta(\hat\rvx_{t_{n+1}}, t_{n+1}) - \vf(\rvx_{t_{n+1}}, t_{n+1}; u)\\
        &= \vf_\vtheta(\rvx_{t_n}, t_n) - \vf_\vtheta(\hat\rvx_{t_n}, t_n) + \vf_\vtheta(\hat\rvx_{t_{n+1}}, t_{n+1}) - \vf_\vtheta(\rvx_{t_{n+1}}, t_{n+1})\\
                &\quad\quad\quad + \vf_\vtheta(\rvx_{t_{n+1}}, t_{n+1}) - \vf(\rvx_{t_{n+1}}, t_{n+1}; u)\\
         &= \vf_\vtheta(\rvx_{t_n}, t_n) - \vf_\vtheta(\hat\rvx_{t_n}, t_n) + \vf_\vtheta(\hat\rvx_{t_{n+1}}, t_{n+1}) - \vf_\vtheta(\rvx_{t_{n+1}}, t_{n+1}) + \ve_{n+1}\\
         &\ldots\\
        &\overset{(iii)}{=} \vf_\vtheta(\rvx_{t_n}, t_n) - \vf_\vtheta(\hat\rvx_{t_n}, t_n) + \vf_\vtheta(\rvx_{T}, T) - \vf_\vtheta(\hat\rvx_{T}, T) + \ve_T.
    \end{align*}
    Here, step $(i)$ follows from the definition of the consistency function, step $(ii)$ is due to Eq. (\ref{eq:learned-cf}), and step $(iii)$ leverages the telescoping nature of the sum.

    Furthermore, since $\vf_\vtheta$ is parameterized such that $\vf_\vtheta(\rvx_{T}, T) = \rvx_{T}$, we have
    \begin{align*}
        \ve_T &= \vf_\vtheta(\rvx_{T}, T) - \vf(\rvx_{T}, T; u) \\ 
        &= \rvx_{T} - \rvx_{T}\\
        &= 0.
    \end{align*}

    Finally, given that $\vf_\vtheta$ is Lipschitz and considering the bound on the global error of the ODE solver:
    \begin{equation*}
        \|\ve_n\|_2 \leq \|\ve_T\|_2 + L\|\rvx_{t_n} - \hat\rvx_{t_n} \|_2 + L\|\rvx_{T} - \hat\rvx_{T} \|_2
        = O((\Delta t)^p).
    \end{equation*}
\end{proof}

\section{Experimental Details}\label{sec:details}
\subsection{Target Distributions}
\paragraph{GMM.}
Here we discuss the parameterization for the Gaussian mixture model with well separated modes. 
We follow the same setting as \citet{zhang2022pis, berner2024dis}, defining the target distribution as follows: 
\begin{equation*}
    \rho(\rvx) = \sum_{m=1}^{M} \alpha_m \mathcal{N}(\rvx; \mu_m, \Sigma_m)
\end{equation*}
Following their prarameterization, we set $M=9$, $\sigma_m = .3 I$, and $(\mu)_{m=1}^M = \{-5, -, 5\} \times \{-5, 0, 5\} \subset \mathbb{R}^2$.

\paragraph{Image.}
We use a normalized grayscale image to create a two-dimensional probability density, following the setup from \citet{wu2020snf}. 

\paragraph{Funnel.}
Following the methodology of \citet{berner2024dis}, we use the funnel distribution introduced from \citet{neal2003slice}. The distribution is defined as follows: 
\begin{equation*}
    \rho(\rvx) = \mathcal{N}(x_1; 0, v^2) \prod_{i=1}^d \mathcal{N}(x_i; 0, e^{x_1})
\end{equation*}
We set $d=10, v=3$. 

We include this benchmark as this is a canonical distribution used for comparing MCMC methods and has been used extensively within the growing field of learned diffusion samplers \citep{berner2024dis, zhang2022pis, vargas2023dds, richter2024improved}.

\paragraph{Many-Well.}
We use the many-well target distribution following the methodology of \citet{berner2024dis}:
\begin{equation*}
\rho(\rvx) = \exp \left(-\sum_{i=1}^m (x_i^2 - \delta) - \frac{1}{2} \sum_{i=m+1}^d x_i^2 \right).
\end{equation*}
For the target distribution labeled as MW-54, we set $d=5$, $m=5$, and $\delta=4$; for the target distribution labeled as MW-52, we set $d=50, m=5, \delta=2$. 

\paragraph{Log Cox Gaussian Process (LGCP).}
The log cox Gaussian process is a popular target distribution for benchmarking sampling methods due to its complexity and high-dimensionality. 
As discussed in \citet{zhang2022pis, chen2025SCLD}, the LGCP distribution is defined as follows: 
\begin{align*}
    \rho(x) = \mathcal{N} (x; \mu, \Sigma) \prod_{i=1}^d \exp \left(x_i y_i -\frac{\exp(x_i)}{d}\right).
\end{align*}
Here, $y$ is a given dataset, and $\mu, \Sigma$ are mean and covariance for some given prior. 
We follow the methodology of \citet{zhang2022pis, arbel2021annealed} for both the dataset and prior distribution.

\subsection{Training Details}
For GMM, image, funnel, and MW54, we train all diffusion samplers until convergence or for 30,000 training iterations. 
For MW52d, we train all samplers for 10,000 training iterations. For LGCP, we train all samplers for 5,000 training iterations. 

For a complete specification of sampler details, see Table \ref{tab:sampler_comparison}. 
For details on the global configurations used across all samplers, see Table \ref{tab:details}.

\begin{table*} 
    \centering
    \renewcommand{\arraystretch}{1.3}  
    \begin{minipage}[t]{.45\textwidth}
    \vspace{0pt} 
        \begin{tabular}{p{4cm} p{4cm}}  
            \toprule
            \multicolumn{2}{c}{\textbf{SCDS}} \\  
            \midrule
            Terminal Time & 1 \\
            SDE & VP SDE \\ 
            Terminal Time & 1 \\
            Time Schedule & Linear \\ 
            Initial Distribution & $\mathcal{N}(0, I)$ with Truncation Quartile of $1e-4$ \\
            Loss Function & Log-Variance, Time Reversal \citep{berner2024dis, richter2024improved}, Self-Consistency\\ 
            \midrule 
            \multicolumn{2}{c}{\textbf{CDDS}} \\  
            \midrule
            Pretrained Generative Ctrl & DIS \\ 
            Consistency Model Train Timesteps & 18 \\ 
            Loss Function & Equation \eqref{eq:cdloss} \\
            \midrule 
            \multicolumn{2}{c}{\textbf{DIS} \citep{berner2024dis}}  \\
            \midrule 
            SDE & VP SDE \\ 
            Loss Function & Log Variance, Time Reversal \\ 
            Terminal Time & 1 \\
            Time Schedule & Linear \\ 
            Initial Distribution & $\mathcal{N}(0, I)$ with Truncation Quartile of $1e-4$ \\
            \midrule
            \multicolumn{2}{c}{\textbf{PIS} \citep{zhang2022pis}}  \\
            \midrule 
            SDE & VE SDE \\ 
            Loss Function & Log Variance \citep{richter2024improved}\\ 
            Terminal Time & 1 \\
            Time Schedule & Linear \\ 
            Initial Distribution & Dirac-Delta  \\  
            \midrule
            \multicolumn{2}{c}{\textbf{DDS} \citep{vargas2023dds}}  \\
            \midrule 
            SDE & VP SDE \\ 
            Loss Function & Log Variance \\ 
            SDE & Exponential SDE \\
            Time Schedule & Cosine \\ 
            Terminal T & 12.8 \\ 
            $\Delta t$ & $.1$ \\
            Initial Distribution & $\mathcal{N}(0, I)$ with Truncation Quartile of $1e-4$ \\
            \bottomrule
        \end{tabular}
        \caption{Diffusion Sampler Configurations}
        \label{tab:sampler_comparison}
    \end{minipage} \hfill
    \begin{minipage}[t]{.45\textwidth}
    \vspace{0pt} 
        \begin{tabular}{p{4cm} p{4cm}}  
            \toprule
            \multicolumn{2}{c}{\textbf{Optimizer Settings}} \\  
            \midrule
            Optimizer & Adam \\ 
            Learning Rate & $.005$ \\ 
            Weight Decay & $1e-7$ \\
            Gradient Clipping & 1 \\
            $\beta_1, \beta_2$ & $.9, .999$ \\
            \midrule 
            \multicolumn{2}{c}{\textbf{Training Settings}} \\  
            \midrule
            Total Iterations & GMM, Image, Funnel, MW54=30,000; MW52=10,000; LGCP=5,000 \\
            Train Time Steps & 128 \\ 
            Batch Size & 2048 \\
            \midrule 
            \multicolumn{2}{c}{\textbf{Model Settings}} \\
            \midrule 
            Number of Layers & 4 \\ 
            Channels & 64 \\
            Time Conditioning & Fourier Time Embeddings \citet{tancik2020fourier} \\ 
            Activation & GeLU \\
            \multicolumn{2}{c}{\textbf{Evaluation Settings}} \\
            \midrule
            Batch Size & $10000$ \\
            Weight Decay & $1e-7$ \\ 
            \bottomrule
        \end{tabular}
        \caption{Global Configurations}
        \label{tab:details}
    \end{minipage}
\end{table*}

\begin{figure}[tb!]
\centering
\includegraphics[width=1.0\linewidth]{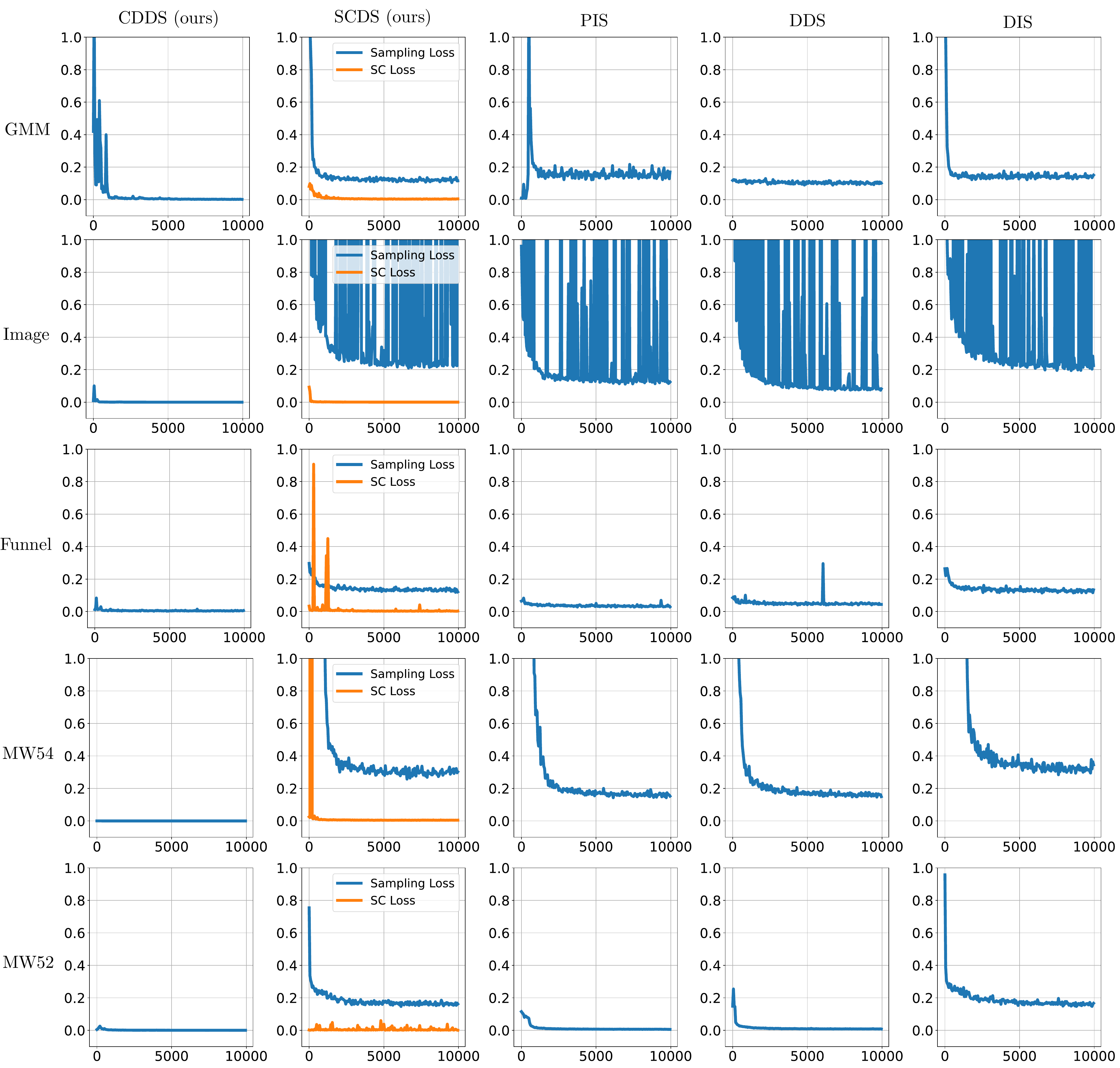}
\caption{Loss curves for the samplers studied in this paper. 
SCDS and CDDS exhibit stable learning across most settings, except for the image target distribution, where all samplers—except CDDS—show instability. 
Notably, the self-consistency loss and the sampling loss remain relatively independent.}
\label{fig:training}
\end{figure}

\end{document}